\documentclass{article} 
\usepackage[a4paper, total={6in, 9in}]{geometry}

\title{
    Hyperparameter Selection Methods for Fitted Q-Evaluation with
    Error Guarantee
}
\usepackage{times}

\author{
    Kohei Miyaguchi\thanks{
        \texttt{miyaguchi@ibm.com}
    }
    \\
    IBM Research - Tokyo\thanks{
        19-21 Hakozaki, Chuo-ku, Tokyo, Japan
    }
}
\input{definition.tex}

\begin{document}

\maketitle

\begin{abstract}
We are concerned with the problem of hyperparameter selection for the fitted Q-evaluation~(FQE).
FQE is one of the state-of-the-art method for offline policy evaluation~(OPE), which is essential to the reinforcement learning without environment simulators.
However, like other OPE methods, FQE is not hyperparameter-free itself and that undermines the utility in real-life applications.
We address this issue by proposing a framework of approximate hyperparameter selection~(AHS) for FQE, which defines a notion of optimality~(called \emph{selection criteria}) in a quantitative and interpretable manner without hyperparameters.
We then derive four AHS methods each of which has different characteristics such as distribution-mismatch tolerance and time complexity.
We also confirm in experiments that the error bound given by the theory matches empirical observations.
\end{abstract}


\section{Introduction}

Offline policy evaluation~(OPE) is an indispensable component of the offline reinforcement learning~(RL),
which is a variant of reinforcement learning with special emphasis on
cost-sensitive real-life applications~\citep{levine2020offline},
such as autonomous vehicles, finance, healthcare and molecular discovery.

Almost all the offline RL algorithms involve their own hyperparameters.
For example, if we employ neural networks in policy learning,
we have to at least decide the network topology~(e.g., number of neurons and layers, to use the residual connection or not, to use the dense connection or the convolution), the activation functions, regularization weights and the optimizers~(e.g., SGD or Adam with their own hyperparameter choices).
The choice of the models such as neural network is also considered to be a hyperparameters.
OPE allows us to optimize or validate the choices over these hyperparameters
based only on offline datasets,
i.e., \emph{without access to environment simulators}.
This is especially useful if test run in the target environment is expensive.

However, with the current form of OPE,
we end up with another hyperparamter selection problem~\citep{paine2020hyperparameter}.
Note that one must employ a higher-order hyperparameter selection scheme to resolve it
and there is no apparent reason to expect such a higher-order problem to be easier than
that of the lower-order problem, i.e., offline RL itself.

In this paper, we seek for the OPE methods
that requires no hyperparameter.
In particular,
we consider a class of OPE algorithms generalizing Fitted Q Evaluation~(FQE)~\citep{le2019batch}
and
derive four hyperparameter selection methods for it
based on a newly introduced framework called \emph{approximate hyperparameter selection~(AHS)}.
Differences in their characteristics such as error guarantee and computational time
are investigated theoretically and empirically.

In Section~\ref{sec:preliminary}, we formally
introduce the notion of OPE, FQE and hyperparameter selection for FQE.
Then, in Section~\ref{sec:theory}, we present the main theoretical results,
namely the AHS framework and a key error bound useful to solve it.
Based on the error bound, in Section~\ref{sec:algorithms},
we derive four concrete algorithms with different error guarantees, computational complexities and time horizons,
corresponding to the first and the second row of Table~\ref{tab:comparison}.
We empirically demonstrate effectiveness and limitation of the derived methods in Section~\ref{sec:experiment}.
Finally, we compare our result with previous ones in Section~\ref{sec:related}
and then present a few concluding remarks in Section~\ref{sec:conclusion}.
\ifdefined\shortversion
    All the proofs of the propositions in this paper is given in the extended version.
\else
\fi

\newcommand{\crt}{\Ccal}
\begin{table}[t]
    \centering
    \newcommand{\mymark}{\heartsuit}
    \caption{
        Comparison of hyperparameter selection methods for FQE.
        The best value is shaded yellow for each column.
        The OPE error of each method is bounded by ${\displaystyle \Ocaltil\rbr{C D \cbr{\min_{X\in\Xcal}\varepsilon(X) + C n^{-1/4}}}}$,
        where $C\coloneqq \sum_{h=1}^{H}\gamma^{h-1}$ is the time constant, $\Xcal$ is the set of hyperparameters and $\Ocaltil(\cdot)$ is hiding the logarithmic factor of $H$, $\abs{\Xcal}$ and $1/\delta$.
        See Section~\ref{sec:method_comparison} and \ref{sec:related} for detailed discussion.
    }
    \label{tab:comparison}
    \newcommand{\cemph}[1]{\colorbox{yellow}{{ #1 }}}
    \renewcommand{\arraystretch}{1.2}
    {\small
    \begin{tabular}{l|cc|c|c}
        \toprule
        \multirow{2}{*}{Method} & \multicolumn{2}{c|}{Off-policy factor $D$}    & \multirow{2}{*}{Error metric $\varepsilon(X)$} & \multirow{2}{*}{Hyperparam.} \\
                                & $H<\infty$ & $H=\infty$ & & \\
        \midrule            
        RM / RM-FP      & \cemph{${\displaystyle \max_{1\le h\le H}\norm{w_h}_{2}}$}    & \cemph{$\norm{w}_{2}$}    & ${ \norm{\Delta X}_2}$                      & \cemph{None} \\ 
        KLM / KLM-FP    & ${\displaystyle \max_{1\le h\le H}\norm{w_h}_{\Fcal_\kappa}}$ & $\norm{w}_{\Fcal_\kappa}$ & \cemph{$\norm{\Delta X}_{\Fcal_\kappa^*}$}    & $\kappa$\\ 
        \cite{feng2019kernel}   & --- &  $\norm{w}_{\Fcal_\kappa}$        & \cemph{$\norm{Q^X-B_\pi Q^X}_{\Fcal_\kappa^*}$} & $\kappa$ \\ 
        \cite{zhang2021towards} & --- &  $\sqrt{C_{\rm BVFT}}\norm{w}_2$  & ${ \norm{Q^X-Q^\pi}_\infty}$                   & ($\epsilon_{\rm dct}$)$^a$\\ 
        \bottomrule
    \end{tabular}
    }
    \vskip 1mm
    \footnotesize{$^a$ A heuristic method for automatic selection of $\epsilon_{\rm dct}$ has been proposed.}
\end{table}

\section{Preliminary}
\label{sec:preliminary}

Let $\Pcal(\Xcal)$ denote the space of probability distributions on $\Xcal$,
where $\Xcal$ is an arbitrary measure space.
The order notation $\Ocal(\cdot)$ is used to hide universal multiplicative constants
in the limit of $n\to\infty$, where $n$ is the data size.
Let $\norm{\cdot}_p$ denotes the $L^p(\mu)$-norm of functions defined over $\Scal\times \Acal$, $p\ge 1$,
where $\mu$, $\Scal$ and $\Acal$ are defined in Section~\ref{sec:preliminary_ope}.
We assume $\Scal\times \Acal$ is a compact measurable space and $\supp(\mu)=\Scal\times \Acal$.
We also assume functions are suitably measurable.

\subsection{Offline Policy Evaluation~(OPE)}
\label{sec:preliminary_ope}

The goal of OPE is to estimate the value
of given policy $\pi$, $J(\pi)$,
with respect to the sequential decision making in the environment of interest $\Mcal$.

The policy value $J(\pi)$ quantifies
the expected rewards obtained within some time horizon
by sequentially taking action according to policy $\pi$.
It is formally defined as
\begin{align*}
    J(\pi)\coloneqq \EE^\pi\sbr{\sum_{h=1}^{H} \gamma^{h-1}r_h},
\end{align*}
where
$1\le H< \infty$\footnote{
    The infinite horizon case is handled later~(Section~\ref{sec:algo_infinite_horizon}). For now, we assume $H<\infty$.
}
and $\gamma\in[0,1]$ respectively denote the time horizon parameter and the discounting factor that determine how far in the future the rewards are taken into account as the value.
The sequence $\cbrinline{r_h}_{h\ge 1}$ denotes
the rewards generated with the policy $\pi$ and the environment $\Mcal$.
The symbol $\EE^\pi$ represents the expectation operator, highlighting the dependency on $\pi$.
Let $\Scal$ and $\Acal$ be the suitably-defined state space and action space, respectively.
We assume
the policy is identified with a state-conditional action distribution $\Scal\ni s\mapsto \pi(s)\in \Pcal(\Acal)$
and
the environment is a Markov decision process~(MDP) $\Mcal\equiv (S_1, R, T)$,
where
$S_1\in\Pcal(\Scal)$,
$R(s,a)\in \Pcal([0, 1])$ and $T(s,a)\in \Pcal(\Scal)$ respectively denote
the initial state distribution,
the conditional reward distribution
and the conditional succeeding-state distribution
given state-action pair $(s,a)\in\Scal\times \Acal$.
Thus, the rewards are subject to the following chain of distributional equations,
$s_1\sim S_1$,
$a_{h}\sim \pi(s_{h})$,
$r_{h}\sim R(s_h,a_h)$,
$s_{h+1}\sim T(s_h,a_h)$,
$h\ge1$.
For convenience, we denote by $P_h$, $1\le h<\infty$, the marginal distribution of $(s_h,a_h)$ induced with $(\Mcal, \pi)$
and by $\nu\coloneqq (1-\gamma)\sum_{h=1}^H\gamma^{h-1}P_h$ its discounted average.

The major constraint of OPE is that the environmental parameters $(T, R)$ are unknown
and $J(\pi)$ must be inferred with an offline dataset $\Dcal$.
We assume the dataset consists of $n$ transition tuples $\Dcal\coloneqq\cbrinline{(\stil_i,\atil_i,\rtil_i,\stil'_i)}_{i=1}^n$
sampled with an unknown query distributions $\mu\in\Pcal(\Scal\times \Acal)$
such that $(\stil_i,\atil_i)\sim \mu$, $\rtil_i\sim R(\stil_i,\atil_i)$, $\stil'_i\sim T(\stil_i,\atil_i)$,
$1\le i\le n$.
We sometimes abuse the notation and write $(\stil_i,\atil_i,\rtil_i,\stil'_i)\sim \mu$
and $\Dcal\sim \mu^n$.

Finally, we introduce a common assumption of OPE,
the condition of sufficient exploration.
\begin{assumption}[Sufficient exploration]
    \label{asm:sufficient_exploration}
    For $1\le h\le H$, the distribution of $(s_h, a_h)$ is absolutely continuous
    with respect to $\mu$,
    i.e., the Radon--Nikodym derivative $w_h(s,a)\coloneqq \frac{\d P_h}{\d \mu}(s,a)$ and
    its discounted average $w(s,a)\coloneqq \frac{\d \nu}{\d \mu}(s,a)$ exist.
\end{assumption}
In other words, it is guaranteed 
the data distribution $\mu$ has positive measure on any measurable events $E\subset \Scal\times \Acal$
that can be happened to the target state-action pairs $(s_h, a_h)$ at some time steps $1\le h\le H$.
Note that this assumption is significantly relaxed if we know a parametric model of MDPs that contains the environment $\Mcal$,
such as linear MDPs~(e.g., Assumption 1 in \cite{duan2020minimax}).
However, we do not assume we know such models in the present study
as our goal is to select the best hyperparameter from data, not from domain knowledge.

\subsection{Fitted Q-Evaluation~(FQE)}
The fitted Q-evaluation is a simple, yet effective OPE algorithm proposed by
\cite{le2019batch}.
It solves a slightly more general problem than OPE,
the estimation of the action-value function.
The action-value function $Q^\pi(s,a)$ quantifies
the value of taking given action $a\in\Acal$ at given state $s\in\Scal$
and then following the policy $\pi$ to make all the subsequent decisions.
It is formally defined as
\begin{align*}
    Q^\pi(s,a)\coloneqq \EE^\pi\sbr{\sum_{h=1}^{H} \gamma^{h-1}r_h \,\middle|\,s_1=s,\,a_1=a}.
\end{align*}
Note that the policy value $J(\pi)$ is computable with $Q^\pi$,
\begin{align*}
    J(\pi)=J(Q^\pi)\coloneqq \EE[Q^\pi(s_1,\pi(s_1))].
\end{align*}

FQE is derived from the recursive property of $Q^\pi$.
More specifically, the action-value function $Q^\pi$ is known to be satisfying the Bellman equation,
$Q^\pi_h=B_\pi Q^\pi_{h-1}$, $h\ge 1$, where
$Q^\pi_h$ is the action-value function with the time horizon set to $H=h$
and
$B_\pi$ is the Bellman operator given by
\begin{align*}
    (B_\pi f)(s,a)
    &\coloneqq
    \EE\sbr{R(s,a)+\gamma f(s',\pi(s'))\,\middle|\, s'\sim T(s,a)},
\end{align*}
for $f:\Scal\times \Acal\to \RR$,
$s\in\Scal$ and $a\in\Acal$.
This implies by induction
\begin{align*}
    Q^\pi=Q^\pi_H=B_\pi^H 0=\underbrace{B_\pi(B_\pi(\cdots(B_\pi}_{\text{$H$ times}} 0))) .
\end{align*}
Therefore, a natural idea to estimate $Q^\pi$ is
to construct an approximate Bellman operator $X\approx B_\pi$
and then apply it to the zero function $H$ times
to obtain the action-value function estimate $Q^X\coloneqq X^H 0\approx Q^\pi$.
We refer to this abstract procedure as MetaFQE~(Algorithm~\ref{alg:metafqe}).
FQE is derived with one of the most natural implementations of $X$, 
the least-squares regression operator,
\begin{align}
    &X_{\mathrm{FQE}}(\pi,\Dcal;\Qcal): f\mapsto 
    \argmin_{g\in \Qcal}\sum_{(s,a,r,s')\in \Dcal} \abs{r+\gamma f(s',\pi(s')) - g(s,a)}^2
    \label{eq:fqe_operator}
\end{align}
with $\Qcal$ being a hypothetical set of action-value functions.

Note that FQE has a hyperparameter $\Qcal$
that heavily influences the output of the algorithm.
It is usually given as a parametric model of functions 
such as linear functions and neural networks.
Moreover,
practical implementations of the FQE opeartor often involve
a number of hyperparameters other than $\Qcal$
such as regularization terms and optimizers.

\begin{algorithm2e}[t]
    {\small
    \caption{
        Meta Fitted Q-Evaluation~(MetaFQE)
    }
    \label{alg:metafqe}
    \KwInput {Approximate Bellman operator $X$}
    \KwOutput {Action value function estimate $Q^X$}
    $Q^X_0\gets 0$\;
    \For {$h=1,2,...,H$} {
        $Q^X_h\gets X Q^X_{h-1}$\;
    }
    \Return $Q^X\gets Q^X_H$\;
    }
\end{algorithm2e}

\subsection{Hyperparameter Selection for MetaFQE}

Observe that a single hyperparameter configuration of FQE
is corresponding to a single operator $X$.
Hence,
the hyperparameter selection of FQE
is equivalent to select the best operator $X_*$ from given candidates of operators $\Xcal$.
Generalizing this idea,
we first introduce the scope of operators $\Omega$
from which the candidate sets $\Xcal$ are taken.

\begin{definition}[Range-bounded operators]
    Let $C\coloneqq \sum_{h=1}^H \gamma^{h-1}$.
    Let $\Omega$ denote
    the set of all the operators on
    $[0, C]$-valued functions over $\Scal\times \Acal$,
    \begin{align*}
        \Omega\coloneqq
        \cbr{X:[0, C]^{\Scal\times \Acal}\to [0,C]^{\Scal\times \Acal}}.
    \end{align*}
\end{definition}
The restriction on the range boundedness is justified
since
the true action-value functions $Q^\pi_h$, $1\le h\le H$, 
are all bounded to $[0, C]^{\Scal\times \Acal}$.
Note that one can modify any $X$ to satisfy the boundedness 
by composing it with a clipping function,
$\Xtil=\texttt{clip}\circ X$,
where 
$\texttt{clip}(f)(s,a)\coloneqq \max\cbrinline{0,\min\cbrinline{C, f(s,a)}}$,
$s\in\Scal$, $a\in\Acal$.

Our goal is formally defined as solving the following problem.
\begin{problem}[Ideal hyperparameter selection for FQE]
    \label{problem:hyperparameter_selection_for_fqe}
    Given $\pi$, $\Dcal$ and $\Xcal\subset \Omega$ with $|\Xcal|<\infty$,
    find $X_*\in \Xcal$
    such that
    \begin{align}
        \abs{\Delta J(Q^{X_*})}
        = \min_{X\in \Xcal}
        \abs{\Delta J(Q^X)},
        \label{eq:hyperparameter_selection}
    \end{align}
    where
    $Q^X\coloneqq X^H Q$ is the Q-function generated by $X$
    and
    $\Delta J(Q)\coloneqq J(Q)-J(Q^\pi)$ is the OPE error
    associated with $Q:\Scal\times \Acal\to [0,C]$.
\end{problem}
Without loss of generality,
we assume each $X\in\Xcal$ is independent of $\Dcal$.
Although the operators are often learned from the dataset as in FQE,
the independence is guaranteed with the training-validation split $\Dcal=\Dcal_{\mathrm{train}}+\Dcal_{\mathrm{valid}}$. 
The subsequent analyses and discussions are also applicable to this setting
simply by replacing $\Dcal$ with $\Dcal_{\mathrm{valid}}$.

\section{Theoretical Results}
\label{sec:theory}
Problem~\ref{problem:hyperparameter_selection_for_fqe}
cannot be always solved since the OPE error $\absinline{\Delta J(Q^X)}$
is difficult to estimate in general.
To address this issue, we first introduce a relaxation of Problem~\ref{problem:hyperparameter_selection_for_fqe},
namely the \emph{approximate hyperparameter selection}~(AHS) problem.
Then, we present a useful theoretical tool to solve it,
which is heavily exploited later~(in Section~\ref{sec:algorithms}) to derive hyperparameter-selection algorithms.

\subsection{Approximate Hyperparameter Selection Framework}
To define a relaxation of Problem~\ref{problem:hyperparameter_selection_for_fqe},
we first introduce the notions of the selection criteria
and the optimality of choices.

\begin{definition}[Selection criterion]
    \label{def:approximate_criteria}
    A function $\crt:\Omega\to \RR$
    is said to be a selection criterion
    when the following conditions are met.
    \begin{enumerate}
        \item For all $X\in \Omega$,
            $\abs{\Delta J(Q^{X})}\le \crt(X)$.
        \item $\crt(B_\pi)=0$.
    \end{enumerate}
\end{definition}

\newcommand{\subopt}{\mathrm{Subopt}}
\begin{definition}[$\crt$-optimality]
    \label{def:c_optimal_choice}
    Let $\Xcal\subset \Omega$ be a set of candidate operators
    and $\crt$ be any selection criterion.
    Let $\Qhat:\Scal\times \Acal\to [0,C]$ represent a random function.
    We say a function $\Qhat$ is $(\Xcal, \crt)$-optimal,
    or $\crt$-optimal if there is no ambiguity, if and only if
    \begin{align}
        \abs{\Delta J(\Qhat)} \le \min_{X\in\Xcal}\crt(X) + o_P(1),
        \label{eq:ahs_optimality}
    \end{align}
    where $o_P(1)$ denotes a diminishing term, $\PP\cbrinline{\absinline{o_P(1)}>\epsilon}\overset{n\to \infty}{\longrightarrow} 0$ for all $\epsilon>0$.
    Equivalently, $\Qhat$ is $(\Xcal, \crt)$-optimal if and only if it
    achieves asymptotically zero $\crt$-suboptimality in probability,
    $\mathrm{Subopt}(\Qhat;\Xcal,\crt)\overset{P}{\to}0$,
    where the suboptimality is defined as
    \begin{align*}
        \subopt(\Qhat;\Xcal,\crt)
        \coloneqq \max\cbr{0,\, \abs{\Delta J(\Qhat)} - \min_{X\in\Xcal}\crt(X)}.
    \end{align*}
\end{definition}

Now, we are ready to define a relaxation of Problem~\ref{problem:hyperparameter_selection_for_fqe}.
\begin{problem}[$\crt$-approximate hyperparameter selection ($\crt$-AHS)]
    \label{problem:approximate_hyperparameter_selection}
    Let
    $\crt$ be a given selection criterion.
    For given $\pi$, $\Dcal$ and $\Xcal$ with 
    $\absinline{\Xcal}<\infty$,
    find a $\crt$-optimal Q-function $\Qhat$.
\end{problem}

A few remarks follow in order.
Firstly, $\crt$-AHS is in fact a relaxation of Problem~\ref{problem:hyperparameter_selection_for_fqe}.
This is seen from that, in \eqref{eq:ahs_optimality},
we have weakened the solution condition
replacing the RHS of \eqref{eq:hyperparameter_selection}
with a probabilistic upper bound, $\min_{X\in\Xcal} \absinline{\Delta J(Q^X)}\le \min_{X\in\Xcal} \crt(X)+o_P(1)$.
Specifically, all the solutions $X_*$ of Problem~\ref{problem:hyperparameter_selection_for_fqe}
induce $\Qhat=Q^{X_*}$ with zero $\crt$-suboptimality with any $\crt$.

Secondly, the solutions of $\crt$-AHS are asymptotically consistent.
If the candidate set $\Xcal$ happens to contain the true operator $B_\pi$
and $\Qhat$ is $\crt$-optimal,
we have $\Delta J(\Qhat)\to 0$ in probability.
Moreover, the OPE error of $\Qhat$ is exactly characterized
with the $\crt$-suboptimality,
$\absinline{\Delta J(\Qhat)}=\subopt(\Qhat;\Xcal,\crt)$.

Thirdly, even if $\Xcal$ does not contain the true operator,
the OPE error is bounded by
$\absinline{\Delta J(\Qhat)}\le \min_{X\in\Xcal}\crt(X)+o_P(1)$.
Therefore, the asymptotic quality of the selection
depends on the tightness of the criterion $\crt(X)$.

Finally, the values of criteria $\crt(X)$ themselves are not necessarily tractable.
The minimum requirement is that we have an algorithm that gives a $\crt$-optimal Q-function.
In fact, all the algorithms presented in this paper minimize computationally \emph{intractable} criteria.

\subsection{Master Error Bound for AHS}

Now, we show upper bounds useful to solve AHS for FQE.
To this end, we first introduce the notion of the dual norms of $\Fcal$ and the link functions.
\begin{definition}[Dual norm]
    Let $\Fcal\subset L^1(\mu)$ be a Banach space.
    The dual norm of $\Fcal$ for functions is given by
    \begin{align*}
        \norm{g}_{\Fcal^*}
        &\coloneqq \sup_{\norminline{f}_\Fcal\le 1}\EE_{(s,a)\sim \mu}[f(s,a)\,g(s,a)],
    \end{align*}
    where $g:\Scal\times \Acal\to \RR$.
    Moreover, 
    abusing the notation, the dual norm for operators is given by
    \begin{align*}
        \norm{X}_{\Fcal^*}
        &\coloneqq \sup_{f:\Scal\times \Acal\to [0, C]} \norm{Xf}_{\Fcal^*}, \quad X\in\Omega.
    \end{align*}
\end{definition}
\begin{definition}[Link function]
    We say a function $\varphi:\RR_{\ge 0}\to \RR$ is a link function
    if it is nonnegative, nondecreasing, continuous, concave,
    and satisfying $\varphi(0)=0$.
\end{definition}

\begin{proposition}[Master error bound]
    \label{prop:metafqe_error_bound}
    Let $\Fcal$ be a dense Banach subspace of $L^1(\mu)$.
    Let $\Delta X\coloneqq X -B_\pi$ be the Bellman error operator of $X\in\Omega$.
    Then, under Assumption~\ref{asm:sufficient_exploration},
    there exists a link function $\varphi_\Fcal:\RR_{\ge 0}\to\RR_{\ge 0}$ such that
    $\forall y\ge 0$, $\varphi_\Fcal (y) \le C\max_{1\le h\le H}\norm{w_h}_\Fcal y$
    and
    \begin{align}
        \abs{\Delta J(Q^X)}
        &\le
        \varphi_\Fcal \rbr{
            \frac1C\sum_{h=1}^H \gamma^{H-h} \norm{\Delta X Q^X_{h-1}}_{\Fcal^*}
        }
        \label{eq:metafqe_error_bound}
    \end{align}
    for all $X\in\Omega$.
\end{proposition}
\ifdefined\shortversion\else
\begin{proof}
    See Section~\ref{sec:app_proof_metafqe_error_bound}.
\end{proof}
\fi

Proposition~\ref{prop:metafqe_error_bound}
suggests
the RHS of \eqref{eq:metafqe_error_bound}
can be used as a selection criterion as long as $\Fcal$ is dense in $L^1(\mu)$.
We refer to the argument of the link function,
\newcommand{\precrt}{\Ccaltil}
\begin{align*}
    \precrt_\Fcal(X)\coloneqq \frac1C\sum_{h=1}^H \gamma^{H-h} \norm{\Delta X Q^X_{h-1}}_{\Fcal^*},
\end{align*}
as the \emph{precriterion} of $X$ with respect to $\Fcal$.
Since $\varphi_\Fcal$ is a link function,
the minimization of the RHS is possible if the minimization of the precriterion is.
Thus, to solve $\crt$-AHS,
we confine our focus to the construction of upper bounds on the precriterion.

\section{Algorithms}
\label{sec:algorithms}

Proposition~\ref{prop:metafqe_error_bound}
suggests a spectrum of OPE error bounds
corresponding to different error-measuring Banach spaces $\Fcal$.
In general,
there is a trade-off in the choice of $\Fcal$.
If $\Fcal$ is more expressive,
the link function $\varphi_\Fcal$ is smaller
but the precriterion is larger.
To see this, consider two Banach spaces $\Fcal$ and $\Gcal$
such that $\norm{\cdot}_\Fcal\ge \norm{\cdot}_\Gcal$ (i.e., $\Gcal$ is more expressive than $\Fcal$).
Then,
we have
$\precrt_\Fcal( X )\le \precrt_\Gcal( X )$
for all $X\in\Omega$ by the definition of the dual norm
and
$\varphi_\Fcal(y)\ge \varphi_\Gcal(y)$ for all $y\ge 0$
by definition (see the proof of Proposition~\ref{prop:metafqe_error_bound}).
Below, we discuss the algorithms induced by typical choices on $\Fcal$.

\subsection{A Failed Attempt}
The most trivial and most expressive choice of $\Fcal$
is $\Fcal=L^1(\mu)$.
In this case, the link function is explicitly calculated
as
$\varphi_{L^1(\mu)}(y)=y$ for $y\ge 0$.
However, the precriterion $\precrt_{L^1(\mu)}$ is difficult
to estimate or minimize
since the corresponding dual norm is the
essential supremum $\norm{\cdot}_\infty$.

\subsection{Regret Minimization (RM)}
A slightly less expressive space is $\Fcal=L^2(\mu)$.
Note that $L^2(\mu)$ is dense in $L^1(\mu)$.
In this case, the dual space is itself, $\Fcal^*=\Fcal$,
and
the precriterion is the sum of the $L^2(\mu)$-norms $\norminline{\Delta X Q^X_{h-1}}_2$.
As shown below,
the norms are simplified using the squared Bellman loss
\newcommand{\sbl}{\Lcal_{\Dcal,\pi}}
\begin{align*}
    &\sbl( X;f)
    \coloneqq 
    \frac{1}{n} \sum_{(s,a,r,s')\in \Dcal, a'\sim \pi(s')}
    \cbr{
        r+\gamma f(s', a')- ( X f)(s, a)
    }^2.
\end{align*}
\begin{proposition}[Squared-loss representation of $L^2(\mu)$-norm]
    \label{prop:l2_residual_identity}
    For any $f:\Scal\times \Acal\to \RR$, we have
    \begin{align}
        \norm{\Delta X f}_{2}^2
        &=
        \EE\sbr{
            \sbl(X;f) - \sbl(B_\pi;f)
        }.
        \label{eq:l2_residual_identity}
    \end{align}
\end{proposition}
\ifdefined\shortversion\else
\begin{proof}
    See Section~\ref{sec:app_proof_l2_residual_identity}.
\end{proof}
\fi
Note that the identity~\eqref{eq:l2_residual_identity}
cannot be used directly to
evaluate the precriterion $\precrt_{L^2(\mu)}$ since we have the true Bellman operator $B_\pi$
on the RHS, which is unknown.
Instead, we introduce a proxy loss called the Bellman regret,
\newcommand{\regret}{\mathrm{Regret}_{\Dcal,\pi}}
\begin{align*}
    \regret( X;\Xcal,f)
    \coloneqq
    \sbl( X;f) - \min_{A\in \Xcal}\sbl(A;f),
\end{align*}
which substitutes $B_\pi$ with the best approximate operator in $\Xcal$.
The Bellman regret is then used to compute the total regret,
\newcommand{\totregret}{\mathrm{\overline{Regret}}_{\Dcal,\pi}}
\begin{align}
    \totregret( X;\Xcal)
    \coloneqq
    \frac1C\sum_{h=1}^H \gamma^{H-h}
    \sqrt{\regret( X;\Xcal,Q^X_{h-1})},
    \label{eq:bellman_regret_sum}
\end{align}
which approximate the precriterion $\precrt_{L^2(\mu)}(X)$.
We refer to the minimization of the total regret as Regret Minimization~(RM)~(Algorithm~\ref{alg:fqe_rm} in the appendix).
In fact, RM is shown to be optimal with respect to
a selection criterion
    $\crt_2(X)\coloneqq 3\varphi_{L^2(\mu)}\rbrinline{
        \norm{\Delta X}_2 
    }$.

\newcommand{\Xrm}{{\Xhat_{\rm RM}}}
\begin{proposition}[Optimality of RM]
    \label{prop:optimality_of_rm}
    Let $\Xrm\coloneqq\argmin_{X\in\Xcal}\totregret(X;\Xcal)$.
    Then, $Q^\Xrm$ is $\crt_2$-optimal.
    The suboptimality is bounded by
    \begin{align*}
        \subopt(Q^\Xrm;\Xcal,\crt_2)
        &=
        \Ocal\rbr{C^2\max_{1\le h\le H}\norm{w_h}_{2}\rbr{\frac{\ln (H\abs{\Xcal}/\delta)}{n}}^{1/4}}
    \end{align*}
    with probability $1-\delta$, where $\delta\in(0,1)$.
\end{proposition}
\ifdefined\shortversion\else
\begin{proof}
    See Section~\ref{sec:app_proof_optimality_of_rm}.
\end{proof}
\fi

\subsection{Kernel Loss Minimization~(KLM)}
As an even less expressive example,
we take $\Fcal$ as reproducing kernel Hilbert spaces~(RKHS) generated by some kernel function
$\kappa:(\Scal\times \Acal)^2\to\RR$.
For simplicity, we assume some regularities of the kernel $\kappa$.
\begin{assumption}
    \label{asm:normalized_kernel}
    $\kappa(u,u')$ is continuous, symmetric and positive definite
    with respect to $u,u'\in\Scal\times \Acal$.
    Moreover, it is normalized, i.e.,
    $\sup_{u\in\Scal\times \Acal}\abs{\kappa(u,u)}\le 1$.
\end{assumption}
A typical example of such kernels inducing dense subspaces of $L^1(\mu)$
is the Gaussian kernels, $\kappa(u,u')=\exp(-\abs{u-u'}_2^2/\sigma^2)$,
where $\abs{\cdot}_2$ is the $\ell^2$-norm of Euclidean space and $\sigma>0$ is a scale parameter.
The following proposition gives a useful identity of the RKHS precriterion
based on such kernels.
\begin{proposition}[Kernel representation of dual RKHS norms]
    \label{prop:kernel_loss_identity}
    Let $\Fcal_\kappa$ be the RKHS generated by $\kappa$.
    Then, for any $f:\Scal\times \Acal\to \RR$, we have
    \begin{align}
        \norm{f}_{\Fcal_\kappa^*}^2 =\mathop{\EE}_{u,\util\sim \mu}\sbr{\kappa(u,\util)\,f(u)\,f(\util)}.
        \label{eq:RKHS_precriterion_identity}
    \end{align}
\end{proposition}
\ifdefined\shortversion\else
\begin{proof}
    See Section~\ref{sec:app_proof_kernel_loss_identity}.
\end{proof}
\fi

To approximate the dual norm 
based on the data $\Dcal$,
we introduce the kernel Bellman loss,
\newcommand{\kbl}[1]{\Lcal_{#1,\Dcal,\pi}}
\begin{align*}
    \kbl{\kappa}(X;f)
    \coloneqq
    \frac1{n^2}
    \sum_{
        \substack{
            (u,r,u')\in\Dcal_\pi \\
            (\util,\rtil,\util')\in\Dcal_\pi
        }
    }
    \Big[
    \kappa(u,\util)\times 
    \cbr{r+\gamma f(u')-Xf(u)}
    \cbr{\rtil+\gamma f(\util')-Xf(\util)}
    \Big],
\end{align*}
where $(u,r,u')\in\Dcal_\pi$ indicates
$u$ is a state-action pair before transition in $\Dcal$, $r$ is the corresponding reward,
and $u'$ is the pair of the state after transition and the action drawn from $\pi$.
Summing up the kernel Bellman losses, we have
the total kernel loss
\newcommand{\totkbl}[1]{\Lcalbar_{#1,\Dcal,\pi}}
\begin{align}
    \totkbl{\kappa}(X)
    \coloneqq
    \frac1C\sum_{h=1}^H\gamma^{H-h} \sqrt{\kbl{\kappa}(X;Q^X_{h-1})}
    \label{eq:total_kernel_loss}
\end{align}
as an approximation of $\precrt_{\Fcal_\kappa}(X)$.
We refer to the minimization of the total kernel loss as Kernel Loss Minimization~(KLM)~(Algorithm~\ref{alg:fqe_klm} in the appendix).
In fact, KLM is optimal with respect to 
$\crt_\kappa(X)\coloneqq \varphi_{\Fcal_\kappa}(\norm{\Delta X}_{\Fcal_\kappa^*})$,
which is a selection criterion if $\Fcal_\kappa$ is dense in $L^1(\mu)$.

\newcommand{\Xklm}[1]{\Xhat_{\mathrm{KLM}( #1 )}}
\begin{proposition}[Optimality of KLM]
    \label{prop:optimality_of_klm}
    Let $\Xklm{\kappa}\coloneqq \argmin_{X\in\Xcal} \totkbl{\kappa}(X)$.
    Then, $\Xklm{\kappa}$ is $\crt_{\kappa}$-optimal.
    The suboptimality is bounded by
    \begin{align*}
        \subopt(Q^{\Xklm{\kappa}};\Xcal,\crt_\kappa)
        &=
        \Ocal\rbr{C^2\max_{1\le h\le H}\norm{w_h}_{\Fcal_\kappa}\rbr{\frac{\ln (H\abs{\Xcal}/\delta)}{n}}^{1/4}}
    \end{align*}
    with probability $1-\delta$,
    where $\delta\in(0,1)$.
\end{proposition}
\ifdefined\shortversion\else
\begin{proof}
    See Section~\ref{sec:app_proof_optimality_of_kbs_minimization}.
\end{proof}
\fi

\subsection{Efficient Algorithms for Infinite Time Horizon}
\label{sec:algo_infinite_horizon}

Consider the (discounted) infinite time horizon case, where $H=\infty$ and $\gamma <1$.
In this setting, the \naive procedure of MetaFQE is infeasible due to the linear time complexity with respect to $H$.
A possible workaround is to perform the early stopping
exploiting the contraction inequality $\norminline{Q^\pi-Q^{B_\pi}_h}_\infty \le C\gamma^h$, $h\ge 0$, which implies
it only takes $H^*=C\ln (C/\epsilon)$ iterations to bound the additional error due to the early-stopping below $\epsilon$.
With this strategy, the time complexity of both RM and KLM grows at least linearly with respect to the time constant $C=\frac1{1-\gamma}$.

In this section, we derive computationally less expensive variants of RM and KLM based on the fixed-point characterization of the Q-function, i.e., $f=B_\pi f \Leftrightarrow f=Q^\pi$.
We first introduce a modified version of MetaFQE to solve the fixed-point problem.
Then, we present variants of RM and KLM for the hyperparameter selection of the modified MetaFQE, respectively.

\subsubsection{\textoverline{MetaFQE}: FQE for Fixed Point Problem}
The modified algorithm, called \textoverline{MetaFQE}, is shown in Algorithm~\ref{alg:metafqe_fp}.
The difference from the original MetaFQE is that
the number of the iteration $H^*$ is treated as a hyperparameter
and the outputs is the average of Q-functions over time horizons up to $H^*$
unless a fixed point of $X$ is found during the iteration.
\textoverline{MetaFQE} is guaranteed to produce an approximate fixed point of the Bellman operator
if $\Delta X$ is small and $H^*$ is large.
\begin{proposition}[Fixed-point guarantee of \textoverline{MetaFQE}]
    \label{prop:metafqe_fp_guarantee}
    Suppose $H=\infty$ and $\gamma <1$.
    Fix any Banach subspace $\Fcal\subset L^1(\mu)$
    and $\Qbar^X_{H^*}$ be the output of \textoverline{MetaFQE}.
    Then, there exists a constant $M(\Fcal)<\infty$ only depending on $\Fcal$
    such that
    \begin{align}
        \norm{\Qbar^X_{H^*}- B_\pi \Qbar^X_{H^*}}_{\Fcal^*}
        &\le 
        \norm{\Delta X}_{\Fcal^*} + \frac{C}{H^*}M(\Fcal)
        \label{eq:metafqe_fp_error_bound}
    \end{align}
    for all $X\in \Omega$.
\end{proposition}
\begin{proof}
    See Section~\ref{sec:app_proof_metafqe_fp_guarantee}
\end{proof}
Note that, if $\Fcal$ is dense in $L^1(\mu)$, the LHS of \eqref{eq:metafqe_fp_error_bound} is zero if and only if $\Qbar^X_{H^*}=B_\pi \Qbar^X_{H^*}$ almost everywhere,
thereby measuring the error of the fixed-point problem.

\begin{algorithm2e}[t]
    {\small
    \caption{
        \textoverline{MetaFQE}
    }
    \label{alg:metafqe_fp}
    \KwInput {Approximate Bellman operator $X$, iteration number $H^*$}
    \KwOutput {Fixed-point estimate $\Qbar^X_{H^*}$}
    $Q^X_0\gets 0$; $\Qbar^X_{0}\gets 0$\;
    \For {$h=1,2,...,H^*$} {
        $Q^X_h\gets X Q^X_{h-1}$\;
        \lIf{$Q^X_h=Q^X_{h-1}$} {\Return $Q^X_h$}
        $\Qbar^X_{h}\gets (1-\frac1{h}) \Qbar^X_{h-1}+\frac1{h}Q^X_{h}$\;
    }
    \Return $\Qbar^X_{H^*}$\;
}
\end{algorithm2e}

\subsubsection{Regret Minimization for Fixed Point (RM-FP)}

\newcommand{\regretfp}{\mathrm{Regret}^*_{\Dcal,\pi}}
The infinite-horizon variant of RM is given as
the minimization of the fixed-point Bellman regret
$\regretfp(f;\Xcal)\coloneqq \regret(\Id;\Xcal,f)$
over $f\in\cbr{\Qbar^X_{H^*}}_{X\in \Xcal}$,
where $\Id$ is the identity operator.
We refer to this method as Regret Minimization for Fixed Point~(RM-FP),
whose optimality is given as follows.
\newcommand{\Xrmfp}{{\Xhat_{\text{\rm RM-FP}}}}
\begin{proposition}[Optimality of RM-FP]
    \label{prop:optimality_of_rm_fp}
    Suppose $H=\infty$ and $\gamma <1$.
    Let $H^*\ge n^{1/4}$ and $\Xrmfp\coloneqq \argmin_{X\in \Xcal} \regretfp(\Qbar^X_{H^*};\Xcal)$.
    Then, $\Xrmfp$ is $\crt_2$-optimal.
    The suboptimality of $\Xrmfp$ is bounded by
    \begin{align*}
        \subopt(\Qbar^\Xrmfp_{H^*};\Xcal,\crt_2)
        &=
        \Ocal\rbr{C^2\norm{w}_{2}\rbr{\frac{\ln (\abs{\Xcal}/\delta)}{n}}^{1/4}}.
    \end{align*}
    with probability $1-\delta$, where $\delta\in(0,1)$.
\end{proposition}
\ifdefined\shortversion\else
\begin{proof}
    See Section~\ref{sec:app_proof_optimality_of_rm_fp}.
\end{proof}
\fi

\subsubsection{Kernel Loss Minimization for Fixed Point (KLM-FP)}

\newcommand{\kblfp}[1]{\Lcal_{#1,\Dcal,\pi}^*}
The infinite-horizon variant of KLM is
given as the minimization of the fixed-point kernel Bellman loss
$\kblfp{\kappa}(f)\coloneqq \kbl{\kappa}(\Id;f)$
over $f\in\cbr{\Qbar^X_{H^*}}_{X\in \Xcal}$.
The loss is originally proposed by \cite{feng2019kernel}
as the kernel Bellman V-statistic.
We refer to this method as Kernel Loss Minimization for Fixed Point~(KLM-FP),
whose optimality is given as follows.

\newcommand{\Xklmfp}[1]{\Xhat_{\text{\rm KLM-FP}( #1 )}}
\begin{proposition}[Optimality of KLM-FP]
    \label{prop:optimality_of_klm_fp}
    Suppose $H=\infty$ and $\gamma<1$.
    Let $H^*\ge n^{1/4}$ and
    $\Xklmfp{\kappa}\coloneqq \argmin_{X\in\Xcal} \kblfp{\kappa}(\Qbar^X_{H^*})$.
    Then,
    $\Xklmfp{\kappa}$ is $\crt_{\kappa}$-optimal.
    The suboptimality is bounded by
    {\small
    \begin{align*}
        \subopt(\Qbar_{H^*}^{\Xklmfp{\kappa}};\Xcal,\crt_\kappa)
        &=
        \Ocal\rbr{C^2\norm{w}_{\Fcal_\kappa}
        \rbr{\frac{\ln (\abs{\Xcal}/\delta)}{n}}^{1/4}}
    \end{align*}
    }
    with probability $1-\delta$,
    where $\delta\in(0,1)$.
\end{proposition}
\ifdefined\shortversion\else
\begin{proof}
    See Section~\ref{sec:app_proof_optimality_of_klm_fp}.
\end{proof}
\fi

\subsection{Comparison of Algorithms}
\label{sec:method_comparison}

We have derived four hyperparameter selection algorithms for FQE,
namely, RM, KLM, RM-FP and KLM-FP.
In this section, we discuss their properties in a comparative manner
with different perspectives.
See Table~\ref{tab:comparison} for the summary of the theoretical guarantees given by 
Proposition~\ref{prop:optimality_of_rm}, \ref{prop:optimality_of_klm}, \ref{prop:optimality_of_rm_fp}
and \ref{prop:optimality_of_klm_fp}.

\paragraph{Distribution-mismatch tolerance.}
The tolerance to the mismatch of $P_h~(1\le h\le H)$ from $\mu$ is captured with the link function $\varphi_\Fcal$ 
since it is the only quantity in \eqref{eq:metafqe_error_bound} that depends on $P_h$.
The RM family has better tolerance than the KLM family since $\varphi_{L^2(\mu)}\le \varphi_{\Fcal_\kappa}$.
This can be also seen from the off-policy factor $D$ in Table~\ref{tab:comparison}.
For example, consider $\kappa$ as a Gaussian kernel with $\Scal\times \Acal$ being a subset of a Euclidean space.
Then, if $w$ is bounded but discontinuous,
we have $\norminline{w}_2<\infty$ but $\norminline{w}_{\Fcal_\kappa}=\infty$.
Note, however, that KLM and KLM-FP are still consistent with suitable $\kappa$
since they are instances of the AHS framework.

\paragraph{Misspecification tolerance.}
The KLM family has better dependency on $\Delta X$ than the RM family
as $\norm{\cdot}_{\Fcal_\kappa^*}\le \norm{\cdot}_2$.
This implies they are more robust when the hyperparameter candidates $\Xcal$ are poorly specified.
However, the advantage maybe less significant than the distribution-mismatch tolerance
since its effect is bounded even in the worst case, $\norm{\Delta X}_2\le C<\infty$ for all $X\in \Omega$.

\paragraph{Time complexity.}
The time complexities of RM, RM-FP, KLM and KLM-FP are summarized in Table~\ref{tab:time_complexity}
in the appendix.
In either of finite or infinite horizon case,
the RM family dominates the KLM family in terms of the time complexity if
$n\gg K$ and the KLM family dominates if $K\gg n$,
where $K\coloneqq \abs{\Xcal}$
Besides, an advantage of the infinite-horizon methods is
that their time complexities are independent of time constant $C$,
which is not the case with the finite-horizon methods.

\paragraph{Hyperparameters.}
The KLM family requires a hyperparameter, i.e., the kernel $\kappa$.
The choice of the kernel is crucial in the sense that
it determines the tightness of the overall error bound~\eqref{eq:metafqe_error_bound} via $\varphi_{\Fcal_\kappa}$.
However, $\varphi_{\Fcal_\kappa}$ depends on 
unknown quantities such as the data marginal $\mu$ and episode marginal $\cbr{P_h}_{h=1}^H$.
On the other hand, the RM family has no hyperparameter.

\section{Experimental Results}
\label{sec:experiment}
We report experimental results on the comparison of RM and KLM.
We employ InvManagement-v1 from OR-Gym~\cite{hubbs2020or} as the environment.
The offline data $\Dcal$ of $n=480$ is sampled with a completely random policy
following the uniform distribution over the action space.
The evaluation policies $\pi$ are then prepared as $(\epsilon_{\rm eval}:1-\epsilon_{\rm eval})$-mixtures
of the random and expert policies with different $\epsilon_{\rm eval}$.
In Figure~\ref{fig:mean_absolute_error},
it is shown RM outperforms KLMs with most of the kernel configurations
especially with longer horizon $H$ and smaller $\epsilon_{\rm eval}$,
both of which indicate the substantial distribution mismatch is expected.
This matches the theoretical prediction on the distribution-mismatch tolerance and
emphasizes the importance of appropriate kernel selection for KLM.
For further details of the experiment, see Section~\ref{sec:app_experiment}.

\begin{figure}[t]
    \centering
    \includegraphics[width=90mm]{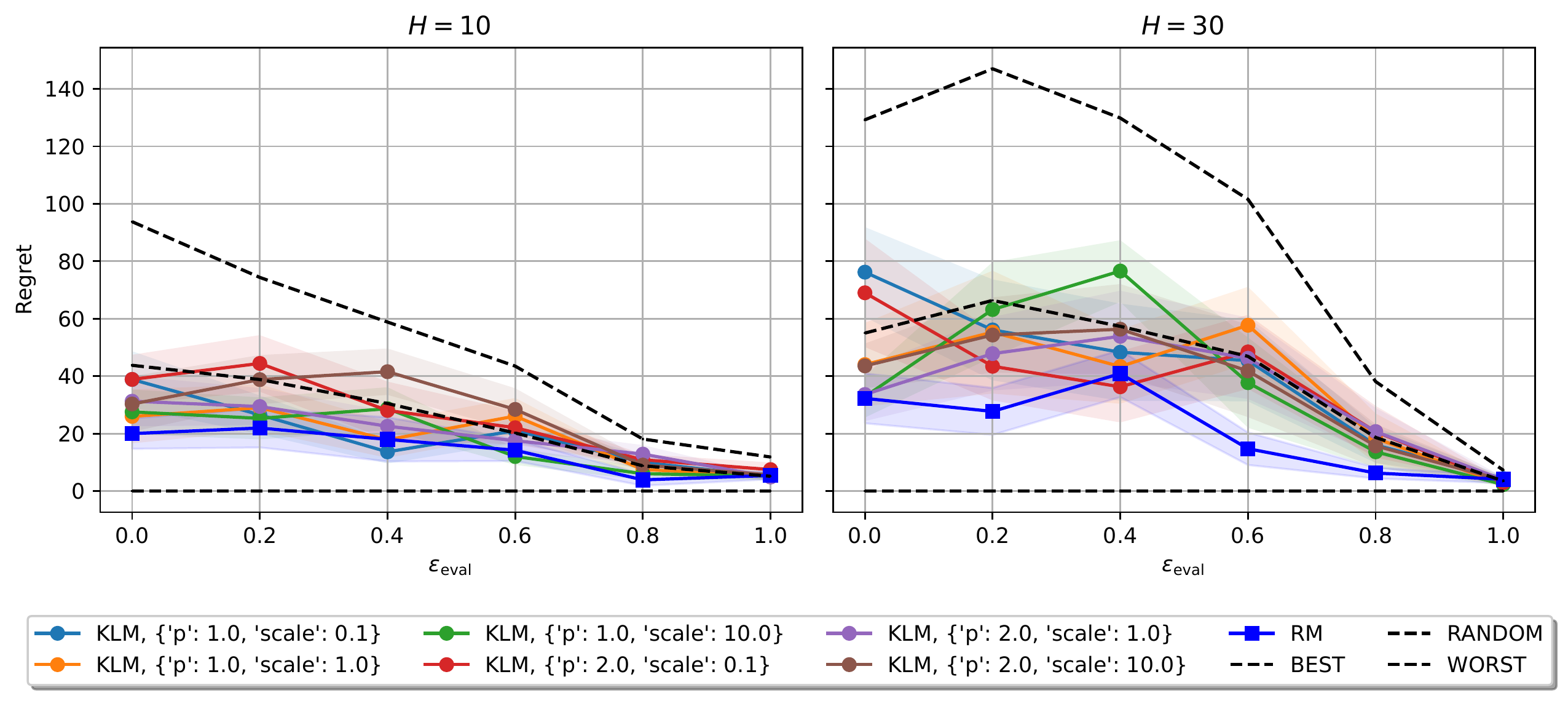}
    \caption{
        The horizontal axis indicates the exploration parameter $\epsilon_{\rm eval}$ of the evaluation policy
        and the vertical axis indicates the excess mean absolute error relative to the best possible choice.
    }
    \label{fig:mean_absolute_error}
\end{figure}

\section{Related Work}
\label{sec:related}

\cite{zhang2021towards} proposed BVFT-PE to solve the hyperparameter selection problem for OPE and 
gave an upper bound on $\norminline{\Qhat-Q^\pi}_2$, which can be translated to the OPE error bound via Proposition~\ref{prop:metafqe_error_bound}.
In comparison, the scope of BVFT-PE is broader than any of the four methods as
it is applicable to any Q-function-estimating OPE algorithms.
Another major difference is in the theoretical assumption; BVFT-PE relies on assumptions much stronger than Assumption~\ref{asm:sufficient_exploration}.
Also, BVFT-PE is not completely hyperparameter-free in terms of the error bound~(Theorem 4, \cite{zhang2021towards}), relying on the oracle choice of $\epsilon_{\rm dct}$.
Another closely related line of research is the loss-minimization formulation of OPE~\citep{baird1995residual,feng2019kernel,dai2018sbeed},
in which OPE is formulated as ordinary optimization problems with objective functions.
The objective functions are readily usable as hyperparameter selection criteria, but there have been no study applying them to minimize the OPE error with theoretical guarantee.
Moreover, these objective functions are either inconsistent~(unable to select the true operator $B_\pi$, e.g., \cite{baird1995residual}) or hyperparameter-dependent by themselves~(e.g., \cite{feng2019kernel,dai2018sbeed}).

At the bottom of Table~\ref{tab:comparison}, we translate these results in a comparative form.\footnote{
    The responsibility of the translation,
    in particular the derivation of the off-policy factor $D$, is ours.
    $C_{\rm BVFT}$ denotes the constant $C$ given by Assumption~1 of \cite{xie2021batch}.
}
Note that RM and RM-FP are the most distribution-mismatch tolerant and the only completely hyperparameter-free methods there.
In terms of the misspecification tolerance, KLM and KLM-FP~(which is equivalent to \textoverline{MetaFQE} + \cite{feng2019kernel}) are the best.

\section{Conclusion}
\label{sec:conclusion}
We have presented hyperparameter selection methods of FQE
and discussed their properties.
In particular, RM and RM-FP are the first hyperparameter selection algorithms
for FQE with \emph{hyperparameter-free error guarantee}.
We have also confirmed in a toy example
that empirical results match the theoretical prediction based on the error guarantee.
The major limitation of our analysis is that it is only applicable to FQE-like algorithms.

Possible future directions include extensions of the KLM methods:
Theoretical justification on a specific kernel choice
and more time-efficient algorithms reducing the factor of $n^2$.


\bibliographystyle{apalike}
\bibliography{reference.bib}

\begin{thebibliography}{}

\bibitem[Baird, 1995]{baird1995residual}
Baird, L. (1995).
\newblock Residual algorithms: Reinforcement learning with function
  approximation.
\newblock In {\em Machine Learning Proceedings 1995}, pages 30--37. Elsevier.

\bibitem[Dai et~al., 2018]{dai2018sbeed}
Dai, B., Shaw, A., Li, L., Xiao, L., He, N., Liu, Z., Chen, J., and Song, L.
  (2018).
\newblock Sbeed: Convergent reinforcement learning with nonlinear function
  approximation.
\newblock In {\em International Conference on Machine Learning}, pages
  1125--1134. PMLR.

\bibitem[Duan et~al., 2020]{duan2020minimax}
Duan, Y., Jia, Z., and Wang, M. (2020).
\newblock Minimax-optimal off-policy evaluation with linear function
  approximation.
\newblock In {\em International Conference on Machine Learning}, pages
  2701--2709. PMLR.

\bibitem[Feng et~al., 2019]{feng2019kernel}
Feng, Y., Li, L., and Liu, Q. (2019).
\newblock A kernel loss for solving the bellman equation.
\newblock {\em arXiv preprint arXiv:1905.10506}.

\bibitem[Feng et~al., 2020]{feng2020accountable}
Feng, Y., Ren, T., Tang, Z., and Liu, Q. (2020).
\newblock Accountable off-policy evaluation with kernel bellman statistics.
\newblock In {\em International Conference on Machine Learning}, pages
  3102--3111. PMLR.

\bibitem[Haarnoja et~al., 2018]{haarnoja2018soft}
Haarnoja, T., Zhou, A., Abbeel, P., and Levine, S. (2018).
\newblock Soft actor-critic: Off-policy maximum entropy deep reinforcement
  learning with a stochastic actor.
\newblock In {\em International conference on machine learning}, pages
  1861--1870. PMLR.

\bibitem[Hubbs et~al., 2020]{hubbs2020or}
Hubbs, C.~D., Perez, H.~D., Sarwar, O., Sahinidis, N.~V., Grossmann, I.~E., and
  Wassick, J.~M. (2020).
\newblock Or-gym: A reinforcement learning library for operations research
  problems.
\newblock {\em arXiv preprint arXiv:2008.06319}.

\bibitem[Le et~al., 2019]{le2019batch}
Le, H., Voloshin, C., and Yue, Y. (2019).
\newblock Batch policy learning under constraints.
\newblock In {\em International Conference on Machine Learning}, pages
  3703--3712. PMLR.

\bibitem[Levine et~al., 2020]{levine2020offline}
Levine, S., Kumar, A., Tucker, G., and Fu, J. (2020).
\newblock Offline reinforcement learning: Tutorial, review, and perspectives on
  open problems.
\newblock {\em arXiv preprint arXiv:2005.01643}.

\bibitem[Luo et~al., 2019]{luo2018algorithmic}
Luo, Y., Xu, H., Li, Y., Tian, Y., Darrell, T., and Ma, T. (2019).
\newblock Algorithmic framework for model-based deep reinforcement learning
  with theoretical guarantees.
\newblock In {\em International Conference on Learning Representations}.

\bibitem[Paine et~al., 2020]{paine2020hyperparameter}
Paine, T.~L., Paduraru, C., Michi, A., Gulcehre, C., Zolna, K., Novikov, A.,
  Wang, Z., and de~Freitas, N. (2020).
\newblock Hyperparameter selection for offline reinforcement learning.
\newblock {\em arXiv preprint arXiv:2007.09055}.

\bibitem[Seno and Imai, 2021]{seno2021d3rlpy}
Seno, T. and Imai, M. (2021).
\newblock d3rlpy: An offline deep reinforcement library.
\newblock In {\em NeurIPS 2021 Offline Reinforcement Learning Workshop}.

\bibitem[Xie and Jiang, 2021]{xie2021batch}
Xie, T. and Jiang, N. (2021).
\newblock Batch value-function approximation with only realizability.
\newblock In {\em International Conference on Machine Learning}, pages
  11404--11413. PMLR.

\bibitem[Yu et~al., 2020]{NEURIPS2020_a322852c}
Yu, T., Thomas, G., Yu, L., Ermon, S., Zou, J.~Y., Levine, S., Finn, C., and
  Ma, T. (2020).
\newblock Mopo: Model-based offline policy optimization.
\newblock In Larochelle, H., Ranzato, M., Hadsell, R., Balcan, M.~F., and Lin,
  H., editors, {\em Advances in Neural Information Processing Systems},
  volume~33, pages 14129--14142. Curran Associates, Inc.

\bibitem[Zhang and Jiang, 2021]{zhang2021towards}
Zhang, S. and Jiang, N. (2021).
\newblock Towards hyperparameter-free policy selection for offline
  reinforcement learning.
\newblock {\em Advances in Neural Information Processing Systems}, 34.

\end{thebibliography}

\newpage
\appendix


\begin{table}
    \caption{
        Comparison of time complexity.
    }
    \label{tab:time_complexity}
    \centering
    \begin{tabular}{cccc}
        \toprule
        RM              & KLM           & RM-FP                     & KLM-FP        \\
        \midrule                                                    
        $\Ocal(HK^2n)$   & $\Ocal(HKn^2)$ & $\Ocal(Kn^{5/4}+K^2n)$    & $\Ocal(Kn^2)$ \\
        \bottomrule
    \end{tabular}
\end{table}

\begin{algorithm2e}[t]
    \caption{
        FQE with Regret Minimization~(FQE-RM)
    }
    \label{alg:fqe_rm}
    \KwInput{Data $\Dcal$, policy $\pi$, operators $\Xcal\coloneqq\cbr{X_k}_{k=1}^K$}
    \KwOutput {Action-value estimate $\Qhat$}
    \For {$k=1,...,K$} {
        ${\rm Reg}_k\gets \totregret(X_k;\Xcal)$ \tcp*{\eqref{eq:bellman_regret_sum}}
    }
    $\khat\gets \argmin_{1\le k\le K} {\rm Reg}_k$\;
    \Return $\texttt{MetaFQE}(X_{\khat})$\;
\end{algorithm2e}

\begin{algorithm2e}[t]
    \caption{
        FQE with Kernel Loss Minimization~(FQE-KLM)
    }
    \label{alg:fqe_klm}
    \KwInput {
        Data $\Dcal$, policy $\pi$, operator candidates $\Xcal=\cbrinline{X_k}_{k=1}^K$, kernel $\kappa$
    }
    \KwOutput {
        Action-value estimate $\Qhat$
    }
    \For {$k=1,...,K$} {
        ${\rm KL}_k\gets \totkbl{\kappa}(X_k)$ \tcp*{\eqref{eq:total_kernel_loss}}
    }
    $\khat\gets \argmin_{1\le k\le K} {\rm KL}_k$\;
    \Return $\texttt{MetaFQE}(X_{\khat})$\;
\end{algorithm2e}

\section{Details on Experiments}
\label{sec:app_experiment}
In both settings, we set $\gamma=1$.
The expert policies are trained via online reinforcement learning with
the soft actor-critic algorithm~\citep{haarnoja2018soft}
so that they attain sufficiently large policy values,
where the implementation is given by d3rlpy~\citep{seno2021d3rlpy}.
The actual reward of the experts are reported in Table~\ref{tab:expert_value}.
The candidate set $\Xcal$ consists of the FQE operators of GBDTRegressor from LightGBM
with different number of trees, ${\tt n\_estimators}\in\cbrinline{2^0, 2^1, ..., 2^7}$,
where the regressors are fitted on i.i.d.~sample of $n=480$ and the other parameters are set to default.
KLM is run with the exponential-type kernel
\begin{align}
    \kappa(u,\util)=\exp\cbr{-\abs{u-\util}_p/\sigma},
    \label{eq:exponential_kernel}
\end{align}
where $\abs{\cdot}_p$ is the $\ell^p$-norm of vectors and $\sigma$ is the scale parameter.
We employed six different kernels with the combination of $p\in\cbr{1,2}$ and $\sigma\in\cbr{10^{-1}, 10^0, 10^1}$.
Moreover, before fed into the kernels or the regressors,
the data $(\stil_i,\atil_i)$ are normalized so that each feature dimension has zero mean and unit variance.

The lines in Figure~\ref{fig:mean_absolute_error} indicates
the excess absolute error, $\absinline{\Delta J(\Qhat)}-\min_{X\in \Xcal} \absinline{\Delta J(Q^X)}$,
averaged over 10 independent runs.
The shaded areas show the estimated standard deviation of the average.

\begin{table}
    \centering
    \caption{The true policy values of the experts~(standard deviation).}
    \label{tab:expert_value}
    \begin{tabular}{c|cc}
        \toprule
        $H$ & 10 & 30 \\
        \midrule
        $J(\pi_{\rm expert})$ & 211.187 (1.146) & 426.978 (0.849) \\
        \bottomrule
    \end{tabular}
\end{table}

\section{Proofs}

\subsection{Proofs for Finite Horizon}

\subsubsection{Proof of Proposition~\ref{prop:metafqe_error_bound}}
\label{sec:app_proof_metafqe_error_bound}


First, we show a useful lemma that characterizes the relationship of the OPE error $\Delta J(Q^X)$ and
the operator error $\Delta X$.
It can be seen as a generalization of the telescoping lemma of model-based RL~(Lemma~4.3 of \cite{luo2018algorithmic}, Lemma~4.1 of \cite{NEURIPS2020_a322852c}).
As opposed to these results, Lemma~\ref{lem:metafqe_error_identity}
allows the approximate operator $X$ to be
those
that admit no model-based interpretation.

\begin{lemma}[Error identity]
    \label{lem:metafqe_error_identity}
    For all $X\in \Omega$,
    \begin{align*}
        \Delta J(Q^X)
        &=
        \sum_{h=1}^H\gamma^{h-1} \EE^\pi\sbr{
            (\Delta X Q^X_{H-h})(s_{h},a_{h})
        }.
    \end{align*}
\end{lemma}
%
\begin{proof}
    Let $\rbar(s,a)\coloneqq \EE[R(s,a)]$ and 
    let $P_\pi$ be the state-transition operator such that
    $(P_\pi f)(s,a)= \EE[f(s',\pi(s'))|s'\sim T(s,a)]$
    for $f:\Scal\times \Acal\to \RR$.
    Then, we have $B_\pi f=\rbar + \gamma P_\pi f$ for all $f:\Scal\times \Acal\to \RR$
    and therefore
    \begin{align*}
        Q^\pi&=B_\pi^H 0 = \sum_{h=1}^H \rbr{\gamma P_\pi}^{h-1} \rbar
    \end{align*}
    by the linearity of $P_\pi$.
    Since $P_\pi$ is linear, we can telescope the sum to get
    \begin{align*}
        Q^\pi&=Q^X_H-Q^X_0 + \sum_{h=1}^H \rbr{\gamma P_\pi}^{h-1} \cbr{
            \rbar + \gamma P_\pi Q^X_{H-h} - Q^X_{H-h+1}
        }.
    \end{align*}
    Note that $\rbar + \gamma P_\pi Q^X_{H-h} - Q^X_{H-h+1}=-\Delta X Q^X_{H-h}$.
    Since we have $Q^X_0=0$ and $Q^X_H=Q^X$ by definition, we also get
    \begin{align*}
        Q^X-Q^\pi
        &=\sum_{h=1}^H \rbr{\gamma P_\pi}^{h-1}
        \Delta X Q^X_{H-h}.
    \end{align*}
    Taking expectation of both sides as in the definition of $J(Q^\pi)$, we get
    \begin{align*}
        \Delta\Jhat(Q^X)
        &=\sum_{h=1}^H \gamma^{h-1}
        \EE\sbr{\rbrinline{P_\pi^{h-1} \Delta  X Q^X_{H-h}}(s_1, a_1)}.
    \end{align*}
    Since $\EE[(P_\pi f)(s_{h}, a_{h})]=\EE[f(s_{h+1}, a_{h+1})]$ for $h\ge 1$
    and $f:\Scal\times \Acal\to \RR$,
    we obtain the desired result by induction.
\end{proof}

Now we are ready to prove Proposition~\ref{prop:metafqe_error_bound}.

\begin{proof}
    Let $\varphi_h(x)\coloneqq \inf_{\norm{f}_\Fcal\le x} \norm{w_h-f}_1$.
    Now, by Lemma~\ref{lem:metafqe_error_identity},
    we have
    \begin{align*}
        \abs{\Delta J(Q^X)}
        &=\abs{
            \sum_{h=1}^H \gamma^{h-1}
            \EE_{(s,a)\sim \mu}\sbr{
                w_h(s,a)\cdot (\Delta X Q^X_{H-h})(s,a)
            }
        }
        \\
        &=
        \sum_{h=1}^H \gamma^{h-1}
        \abs{
            \EE_{(s,a)\sim \mu}\sbr{
                w_h(s,a)\cdot (\Delta X Q^X_{H-h})(s,a)
            }
        }
        .
    \end{align*}
    Fix any $x>0$ and $f\in \Fcal$ satisfying $\norm{f}_\Fcal\le x$.
    Each summand is bounded as
    \begin{align*}
        &\abs{\EE_{(s,a)\sim \mu}\sbr{
                w_h(s,a)\cdot (\Delta X Q^X_{H-h})(s,a)
        }}
        \\
        &=
        \abs{\EE_{(s,a)\sim \mu}\sbr{
                (w_h-f+f)(s,a)\cdot (\Delta X Q^X_{H-h})(s,a)
        }}
        \\
        &\le
        \norm{\Delta X Q^X_{H-h}}_\infty
        \norm{w_h-f}_1
        +
        \abs{\EE_{(s,a)\sim \mu}\sbr{
                f(s,a)\cdot (\Delta X Q^X_{H-h})(s,a)
        }}
        &(\text{H\"older's inequality})
        \\
        &\le
        C \norm{w_h-f}_1
        +
        x \norm{\Delta X Q^X_{H-h}}_{\Fcal^*}.
    \end{align*}
    The last inequality is owing to the boundedness of the range of $X\in\Omega$.
    Taking the infimum over $f$, we get
    \begin{align*}
        \abs{\EE_{(s,a)\sim \mu}\sbr{
                w_h(s,a)\cdot (\Delta X Q^X_{H-h})(s,a)
        }}
        &\le
        C \varphi_h(x)
        +
        x \norm{\Delta X Q^X_{H-h}}_{\Fcal^*}
    \end{align*}
    for all $x>0$.
    Putting it back to the summation, we get
    \begin{align*}
        \abs{\Delta J(Q^X)}
        &\le \sum_{h=1}^H \gamma^{h-1}
        \cbr{
            C \varphi_h(x) + x \norm{\Delta X Q^X_{H-h}}_{\Fcal^*}
        }
        \\
        &=
        C \sum_{h=1}^H \gamma^{h-1}\varphi_h(x) +
        x \sum_{h=1}^H \gamma^{h-1}\norm{\Delta X Q^X_{H-h}}_{\Fcal^*}.
    \end{align*}
    We get the desired result by
    taking the infimum over $x>0$
    and defining
    \begin{align*}
        \varphi_\Fcal(y)&\coloneqq C\inf_{x>0}
        \cbr{
            \sum_{h=1}^H \gamma^{h-1}\varphi_h(x) + x y
        }.
    \end{align*}

    The nonnegativity and monotonicity of $\varphi_\Fcal(y)$, $y\ge 0$,
    are trivial from the definition.
    Note that it is concave also by the definition, which implies the continuity except on the boundary $y=0$.
    Therefore, it suffices to show the boundary condition
    $\lim_{y\downarrow0}\varphi_\Fcal(y)=\varphi_\Fcal(0)=0$.
    In fact, it is a direct consequence of $\varphi_h(x)\to 0$ as $x\to \infty$, $1\le h\le H$,
    which is the case since
    $\Fcal$ is dense in $L^1(\mu)$ and $w_h\in L^1(\mu)$ by the assumptions.

\end{proof}

\subsubsection{Proof of Proposition~\ref{prop:l2_residual_identity}}
\label{sec:app_proof_l2_residual_identity}
\begin{proof}
    Let $\ztil_i(X,f)\coloneqq \rbar_i +\gamma f(\stil'_i, \pi(\stil'_i)) - (Xf)(\stil_i, \atil_i)$, 
    $m_i(X,f)\coloneqq \EE\sbr{\ztil_i(X,f)\middle|\stil_i, \atil_i}$,
    and 
    $v_i(f)\coloneqq \EE\sbr{\cbrinline{\rbar_i +\gamma f(\stil'_i, \pi(\stil'_i))}^2 \middle|\stil_i, \atil_i}$,
    $1\le i\le n$.
    Observe that $m_i^2(X,f)=-(\Delta Xf)(\stil_i, \atil_i)$
    and therefore $\EE[m_i^2(X,f)]=\norm{\Delta Xf}_2^2$ for $1\le i\le n$.
    Thus,
    \begin{align*}
        \EE\sbr{\sbl(X;f)}
        &=
        \EE\sbr{\frac1n \sum_{i=1}^n \ztil_i^2(X,f)}
        \\
        &=
        \EE\sbr{\ztil_1^2(X,f)}
        \\
        &=
        \EE\sbr{m_1^2(X,f)} + \EE\cbr{\ztil_1(X,f) - m_1(X,f)}^2
        \\
        &=
        \norm{\Delta Xf}_2^2 + \EE\sbr{v_1^2(f)}.
    \end{align*}
    In particular, we have 
    \begin{align*}
        \EE\sbr{\sbl(B^\pi;f)}
        &=
        \EE\sbr{v_1^2(f)}
    \end{align*}
    since $\Delta X=0$ if $X=B_\pi$.
    Taking the difference of the above equations,
    we get the desired result.
\end{proof}

\subsubsection{Proof of Proposition~\ref{prop:optimality_of_rm}}
\label{sec:app_proof_optimality_of_rm}
First, we show the following lemma
showing the regret is a good estimate of $\norminline{\Delta X f}_2$
if the set $\Xcal$ well approximates the true operator $B_\pi$ in a collective sense.

\begin{lemma}
    \label{lem:approximate_l2_dual_norm_with_regret}
    Let $\delta\in(0,1)$, $f:\Scal\times \Acal\to [0,C]$,
    $\Xcal\subset \Omega$ and $X\in \Xcal$.
    Then, we have
    \begin{align*}
        \abs{\norm{\Delta X f}_2 - \sqrt{\regret(X;\Xcal,f)}}
        &\le
        \min_{A\in \Xcal} \norm{\Delta A}_2 + C\rbr{\frac{2\ln (2\abs{\Xcal}/\delta)}{n}}^{\frac14}
    \end{align*}
    with probability $1-\delta$.
\end{lemma}
\begin{proof}
    Let $f:\Scal\times \Acal\to [0, C]$.
    With Proposition~\ref{prop:l2_residual_identity}, observe
    \begin{align*}
        \abs{\norm{\Delta X f}_2^2 - \regret(X;\Xcal,f)}
        &\le 
        \underbrace{\abs{\EE\sbr{\sbl(X;f)}-\sbl(X;f)}}_{\text{(A)}} +
        \\
        &\quad 
        \underbrace{\abs{\min_{A\in \Xcal}\sbl(X;f) - \min_{A\in \Xcal}\EE\sbr{\sbl(A;f)}}}_{\text{(B)}} +
        \\
        &\quad 
        \underbrace{\abs{\min_{A\in \Xcal}\EE\sbr{\sbl(A;f)-\sbl(B_\pi;f)}}}_{\text{(C)}}.
    \end{align*}
    By Hoeffding's inequality, we have 
    \begin{align*}
        \abs{\EE\sbr{\sbl(X;f)}-\sbl(X;f)}\le C^2\sqrt{\frac{\ln (2/\delta)}{2n}}
    \end{align*}
    with probability $1 - \delta$ for $X\in\Omega$.
    Thus, by taking union bound with $X\in\Xcal$,
    \begin{align*}
        &
        \text{(A)}\le C^2\sqrt{\frac{\ln (2K/\delta)}{2n}},
        &
        \text{(B)}\le C^2\sqrt{\frac{\ln (2K/\delta)}{2n}},
    \end{align*}
    with probability $1-\delta$, where $K\coloneqq \abs{\Xcal}$.
    As for (C), we have
    \begin{align*}
        \text{(C)}=\min_{A\in\Xcal} \norm{\Delta A f}_2^2
        \le \min_{A\in \Xcal} \norm{\Delta A}_2^2.
    \end{align*}
    by Proposition~\ref{prop:l2_residual_identity}.
    Combining the upper bounds on (A), (B) and (C), we get
    \begin{align*}
        \abs{\norm{\Delta X f}_2^2 - \regret(X;\Xcal,f)}
        &\le
        \min_{A\in \Xcal} \norm{\Delta A}_2^2 + 2C^2\sqrt{\frac{\ln (2\abs{\Xcal}/\delta)}{2n}}
    \end{align*}
    with probability $1-\delta$.
    The desired result is obtained by the fact $\absinline{\sqrt{a}-\sqrt{b}}\le \sqrt{\absinline{a-b}}$,
    $a,b\ge 0$.
\end{proof}

Lemma~\ref{lem:approximate_l2_dual_norm_with_regret} immediately yields the following corollary.

\begin{corollary}
    \label{cor:l2_precriterion_estimation_error}
    Let $\delta\in(0,1)$ and $\Xcal\subset \Omega$.
    Then, we have
    \begin{align*}
        \abs{
            \precrt_{L^2(\mu)}( X )-\totregret( X;\Xcal)
        }
        &\le
        \min_{A\in\Xcal} \norm{\Delta A}_2
        + C\rbr{\frac{2\ln (2H\abs{\Xcal}^2/\delta)}{n}}^{1/4}
    \end{align*}
    with probability $1-\delta$
    for all $X\in\Xcal$ simultaneously.
\end{corollary}
\begin{proof}
    Applying
    Lemma~\ref{lem:approximate_l2_dual_norm_with_regret} with
    the union bound over $X\in \Xcal$ and $f\in \cbrinline{Q^X_{h-1}}_{h=1}^H$,
    we get the desired result, i.e.,
    \begin{align*}
        \abs{\precrt_{L^2(\mu)}(X) - \totregret(X;\Xcal)}
        &\le 
        \frac1C\sum_{h=1}^H \gamma^{H-h}
        \abs{\norm{\Delta X Q^X_{h-1}}_2 - \sqrt{\regret(X;\Xcal,Q^X_{h-1})}}
        \\
        &\le 
        \min_{A\in \Xcal} \norm{\Delta A}_2 +
        C\rbr{
            \frac{2\ln (2K^2H/\delta)}{n}
        }^{\frac14}
    \end{align*}
    for all $X\in \Xcal$
    with probability $1-\delta$.
\end{proof}

Now we are ready to prove Proposition~\ref{prop:optimality_of_rm}.
\begin{proof}
    By Corollary~\ref{cor:l2_precriterion_estimation_error},
    \begin{align*}
        &\precrt_{L^2(\mu)}(\Xrm) - \min_{X\in\Xcal} \precrt_{L^2(\mu)}(X)
        \\
        &\le 
        \totregret(\Xrm;\Xcal) - \min_{X\in\Xcal} \totregret(X;\Xcal) + 
        2
        \cbr{
            \min_{X\in\Xcal} \norm{\Delta X}_2
            + C\rbr{\frac{2\ln (2H\abs{\Xcal}^2/\delta)}{n}}^{1/4}
        }
        \\
        &=
        2
        \cbr{
            \min_{X\in\Xcal} \norm{\Delta X}_2
            + C\rbr{\frac{2\ln (2H\abs{\Xcal}^2/\delta)}{n}}^{1/4}
        },
    \end{align*}
    where the last equality is owing to the definition of $\Xrm$.
    Since
        $\precrt_{L^2(\mu)}(X)
        \le \norm{\Delta X}_2$
    for $X\in \Omega$,
    we have
    \begin{align*}
        \min_{X\in\Xcal} 
        \precrt_{L^2(\mu)}(X)
        \le \min_{X\in\Xcal} \norm{\Delta X}_2.
    \end{align*}
    Combining the above, we get
    \begin{align*}
        \precrt_{L^2(\mu)}(\Xrm)
        &\le
            3\min_{X\in\Xcal}\norm{\Delta X}_2
            + 2C\rbr{\frac{2\ln (2H\abs{\Xcal}^2/\delta)}{n}}^{1/4}
    \end{align*}
    with probability $1-\delta$.
    Applying $\varphi_{L^2(\mu)}(\cdot)$ on both sides, we further get
    \begin{align*}
        &\abs{\Delta J(Q^{\Xrm})}
        \\
        &\le \varphi_{L^2(\mu)}(\precrt_{L^2(\mu)}(\Xrm))
        \\
        &\le \varphi_{L^2(\mu)}\rbr{
                3\min_{X\in\Xcal} \norm{\Delta X}_2
                + 2C\rbr{\frac{2\ln (2H\abs{\Xcal}^2/\delta)}{n}}^{1/4}
        }
        \\
        &\le \min_{X\in\Xcal}\crt_2(X) +
        \varphi_{L^2(\mu)}\rbr{
            2C\rbr{\frac{2\ln (2H\abs{\Xcal}^2/\delta)}{n}}^{1/4}
        }
    \end{align*}
    where the last inequality follows from the monotonicity and the concavity of
    $\varphi_{L^2(\mu)}$.
    The desired result follows from
    $\varphi_\Fcal(y)\le C\max_{1\le h\le H} \norm{w_h}_\Fcal y$.
\end{proof}

\subsubsection{Proof of Proposition~\ref{prop:kernel_loss_identity}}
\label{sec:app_proof_kernel_loss_identity}
\begin{proof}
    It suffices to show the first identity.
    By Mercer's theorem, there exist
    a orthonormal basis $\cbrinline{e_j}_{j=1}^\infty$
    of $L^2(\mu)$ and
    a sequence of positive numbers
    $\cbrinline{\sigma_j}_{j=1}^\infty$
    such that
    \begin{align}
        \kappa(u,\util)=\sum_{j=1}^\infty \sigma_j e_j(u)e_j(\util),\quad u,\util\in \Scal\times \Acal,
        \label{eq:mercers_theorem}
    \end{align}
    where the convergence is uniform.
    Observe that any $f\in\Fcal_\kappa$ is decomposed as
    \begin{align*}
        f(u)
        &=\sum_{m=1}^\infty \alpha_m\kappa(u,u_m), \quad u\in \Scal\times \Acal,
    \end{align*}
    for some $\alpha_m\in \RR$ and $u_m\in\Scal\times \Acal$ ($m=1,2,...$),
    which implies by \eqref{eq:mercers_theorem}
    \begin{align*}
        f(u)
        &=\sum_{j=1}^\infty \beta_j e_j(u),
    \end{align*}
    where $\beta_j\coloneqq \sigma_j \sum_{m=1}^\infty \alpha_m e_j(u_m)$.
    The RKHS norm is accordingly decomposed
    \begin{align*}
        \norm{f}_{\Fcal_\kappa}^2
        &
        =\inner{f}{f}_{\Fcal_\kappa}
        \\
        &
        =\sum_{m=1}^\infty\sum_{\mtil=1}^\infty \alpha_m \alpha_{\mtil} \kappa(u_m,u_{\mtil})
        \\
        &
        =\sum_{m=1}^\infty\sum_{\mtil=1}^\infty \alpha_m \alpha_{\mtil} \sum_{j=1}^\infty \sigma_j e_j(u_m)e_j(u_{\mtil})
        \\
        &=
        \sum_{j=1}^\infty \sigma_j
        \rbr{\sum_{m=1}^\infty \alpha_m e_j(u_m)}^2
        \\
        &=
        \sum_{j=1}^\infty \frac{\beta_j^2}{\sigma_j}.
    \end{align*}
    Let $\Bcal\coloneqq\cbrinline{\cbr{\beta_j}_{j=1}^\infty:\sum_{j=1}^\infty\beta_j^2/\sigma_j\le 1}$
    be the unit ball of the coefficients with respect to $\Fcal_\kappa$.
    Thus, the dual norm is written as
    \begin{align*}
        \norm{g}_{\Fcal_\kappa^*}^2
        &=\rbr{\sup_{\cbr{\beta_j}_{j=1}^\infty\in \Bcal}
        \sum_{j=1}^\infty \beta_j \EE\sbr{e_j(u)g(u)}}^2
        \\
        &= \sum_{j=1}^\infty \sigma_j \cbr{\EE_{u\sim \mu}\sbr{e_j(u)g(u)}}^2
        \\
        &= \mathop{\EE}_{u,\util\sim \mu}\sbr{\sum_{j=1}^\infty \sigma_j e_j(u)e_j(\util)g(u)g(\util)}
        \\
        &= \mathop{\EE}_{u,\util\sim \mu}\sbr{\kappa(u,\util)g(u)g(\util)}.
    \end{align*}
\end{proof}

\subsubsection{Proof of Proposition~\ref{prop:optimality_of_klm}}
\label{sec:app_proof_optimality_of_kbs_minimization}

First, we introduce the concentration result of the kernel Bellman loss,
originally shown by \cite{feng2020accountable}.
\begin{lemma}
    \label{lem:approximate_rkhs_dual_norm_with_kbl}
    For any $f:\Scal\times \Acal\to [0, C]$,
    \begin{align*}
        &\abs{\norm{\Delta X f}_{\Fcal_\kappa^*} - \sqrt{\kbl{\kappa}(X;f)}}
        \le
        C\rbr{\frac{4(1\vee \ln 2/\delta)}{n}}^{\frac14},
    \end{align*}
    where $a\vee b\coloneqq \max\cbr{a, b}$.
\end{lemma}
\begin{proof}
    Proposition 3.1 of \cite{feng2020accountable} shows
    \begin{align*}
        &\abs{\kbl{\kappa}(X;f) - \norm{\Delta X f}_{\Fcal_\kappa^*}^2}
        \le
        2C^2\sqrt{\frac{1\vee \ln 2/\delta}{n}},
    \end{align*}
    which, together with the concavity of $x\mapsto \sqrt{x}$, implies the desired result.
\end{proof}

Now we are ready for the proof.

\begin{proof}
    Taking the union bound with Lemma~\ref{lem:approximate_rkhs_dual_norm_with_kbl},
    we have, with probability $1-\delta$,
    \begin{align*}
        &\abs{\totkbl{\kappa}(X) - \precrt_{\Fcal_\kappa}(X)}
        \le
        C \rbr{\frac{4(1\vee \ln (2H\abs{\Xcal}/\delta))}{n}}^{1/4}
    \end{align*}
    for simultaneously all $X\in\Xcal$,
    which implies
    \begin{align*}
        \precrt_{\Fcal_\kappa}(\Xklm{\kappa}) - \max_{X\in \Xcal} \precrt_{\Fcal_\kappa}(X)
        & \le 
        2C \rbr{\frac{4(1\vee \ln (2H\abs{\Xcal}/\delta))}{n}}^{1/4}.
    \end{align*}
    Therefore, Proposition~\ref{prop:metafqe_error_bound} yields
    \begin{align*}
        \subopt(\Xklm{\kappa};\Xcal,\crt_\kappa)
        &=\max\cbr{0,\, \abs{\Delta J(Q^{\Xklm{\kappa}})} - \min_{X\in\Xcal}\crt(X)}
        \\
        &\le
        \max\cbr{0,\, \crt_\kappa(\Xklm{\kappa}) - \min_{X\in\Xcal}\crt_\kappa(X)}
        \\
        &\le
        \varphi_{\Fcal_\kappa}\rbr{
            2C \rbr{\frac{4\ln (1\vee 2H\abs{\Xcal}/\delta)}{n}}^{1/4}
        },
    \end{align*}
    which, together with $\varphi_\Fcal(y)\le C\max_{1\le h\le H}\norm{w_h}_\Fcal y$,
    implies the desired result.
\end{proof}

%

\subsection{Proofs for Infinite Horizon}

Let us first introduce some useful lemmas.
\begin{lemma}[Error identity, fixed-point form]
    \label{lem:metafqe_fp_error_identity}
    Suppose $H=\infty$ and $\gamma<1$.
    Then,
    \begin{align*}
        \Delta J(Q)
        &=
        C \EE_{(s,a)\sim \nu}\sbr{
            (Q-B_\pi Q)(s,a)
        }
    \end{align*}
    for all $Q:\Scal\times \Acal\to \RR$.
\end{lemma}
\begin{proof}
    Fix $Q:\Scal\times \Acal\to \RR$.
    Take $X\in \Omega$ such that
    $Xf=Q$ for all $f:\Scal\times \Acal\to [0,C]$.
    Then, applying Lemma~\ref{lem:metafqe_error_identity} to $X$ and taking the limit of $H\to \infty$
    yields the desired result.
\end{proof}

\begin{lemma}[Master error bound, fixed-point form]
    \label{lem:metafqe_fp_error_bound}
    Suppose $H=\infty$ and $\gamma<1$.
    Let $\Fcal\subset L^1(\mu)$ be a Banach space.
    Then, under Assumption~\ref{asm:sufficient_exploration},
    there exists a function $\varphi_\Fcal:\RR_{\ge 0}\to\RR_{\ge 0}$ such that
    $\forall y\ge 0$, $\varphi_\Fcal (y) \le C\norm{w}_\Fcal y$
    and
    \begin{align}
        \abs{\Delta J(Q)}
        &\le
        \varphi_\Fcal \rbr{
            \norm{Q-B_\pi Q}_{\Fcal^*}
        }
        \label{eq:metafqe_error_bound_fixed_point}
    \end{align}
    for all $Q:\Scal\times \Acal\to \RR$.
    Moreover,
    if $\Fcal$ is dense in $L^1(\mu)$,
    $\varphi_\Fcal$ is a link function.
\end{lemma}
\begin{proof}
    It is proved similarly as Proposition~\ref{prop:metafqe_error_bound}.
    The difference is to make sure the existence of $\varphi_\Fcal$
    since the limit of the link functions of the finite horizon case may not exist.
    Let $\varphibar(x)\coloneqq C\inf_{\norm{f}_\Fcal\le x} \norm{w-f}_1$.
    Now, by Proposition~\ref{lem:metafqe_fp_error_identity},
    we have
    \begin{align*}
        \abs{\Delta J(Q)}
        &=C\abs{
            \EE_{(s,a)\sim \mu}\sbr{
                w(s,a)\cdot (Q-B_\pi Q)(s,a)
            }
        }
        \\
        &=
        C\abs{\EE_{(s,a)\sim \mu}\sbr{
                (w-f+f)(s,a)\cdot (Q-B_\pi Q)(s,a)
        }}
        \\
        &\le
        C\norm{w-f}_1
        \norm{Q-B_\pi Q}_\infty
        +
        C\abs{\EE_{(s,a)\sim \mu}\sbr{
                f(s,a)\cdot (Q-B_\pi Q)(s,a)
        }}
        &(\text{H\"older's inequality})
        \\
        &\le
        C^2 \norm{w-f}_1
        +
        Cx \norm{Q-B_\pi Q}_{\Fcal^*}
    \end{align*}
    for all $f\in \Fcal$ satisfying $\norm{f}_\Fcal\le x$.
    The last inequality is owing to the boundedness of the range of $X\in\Omega$.
    Taking the infimum over $f$, we get
    \begin{align*}
        \abs{\Delta J(Q)}
        &\le
        C
        \cbr{
            \varphibar(x)
            +
            x \norm{Q-B_\pi Q}_{\Fcal^*}
        }
    \end{align*}
    for all $x>0$.
    We get the desired result by
    further taking the infimum over $x>0$
    and defining
    \begin{align*}
        \varphi_\Fcal(y)&\coloneqq C\inf_{x>0}
        \cbr{
            \varphibar(x) + x y
        }.
    \end{align*}
    That $\varphi_\Fcal$ is a link function is proved in the same way as Proposition~\ref{prop:metafqe_error_bound}.
\end{proof}

\subsubsection{Proof of Proposition~\ref{prop:metafqe_fp_guarantee}}
\label{sec:app_proof_metafqe_fp_guarantee}
\begin{proof}
    First, suppose we found a fixed point of $X$, i.e.,
    there exists $1\le h\le H^*$ such that $\Qbar^X_{H^*}=Q^X_h=Q^X_{h+1}$.
    Then,
    \begin{align*}
        \norm{\Qbar^X_{H^*}- B_\pi \Qbar^X_{H^*}}_{\Fcal^*}
        &=
        \norm{Q^X_{h+1}- B_\pi Q^X_h}_{\Fcal^*}
        \\
        &=
        \norm{\Delta X Q^X_h}_{\Fcal^*}
        \\
        &\le
        \norm{\Delta X}_{\Fcal^*},
    \end{align*}
    which yields the desired inequality.
    In the other case, we have $\Qbar^X_{H^*}=\frac1{H^*} \sum_{h=1}^{H^*}Q^X_h$.
    Therefore,
    \begin{align*}
        \norm{\Qbar^X_{H^*}- B_\pi \Qbar^X_{H^*}}_{\Fcal^*}
        &=
        \norm{\frac1{H^*} \sum_{h=1}^{H^*}Q^X_h- B_\pi \frac1{H^*} \sum_{h=1}^{H^*}Q^X_h}_{\Fcal^*}
        \\
        &=
        \frac1{H^*} \norm{\sum_{h=1}^{H^*}\Delta X Q^X_h+Q^X_{1}-Q^X_{H^*+1}}_{\Fcal^*}
        &(\text{$B_\pi$ is affine})
        \\
        &\le 
        \norm{\Delta X}_{\Fcal^*} + \frac{1}{H^*} \norm{Q^X_{1}-Q^X_{H^*+1}}_{\Fcal^*}.
    \end{align*}
    Now, let $M(\Fcal)\coloneqq \sup_{\norminline{f}_{\Fcal}\le 1}\norminline{f}_1$, which is finite.
    Then, for all $g\in \Fcal^*$,
    \begin{align*}
        \norm{g}_{\Fcal^*}
        &= \sup_{\norm{f}_\Fcal\le 1} \EE_{(s,a)\sim \mu}\sbr{f(s,a) g(s,a)}
        \\
        &\le \sup_{\norm{f}_1\le M(\Fcal)} \EE_{(s,a)\sim \mu}\sbr{f(s,a) g(s,a)}
        \\
        &= M(\Fcal)\sup_{\norm{f}_1\le 1} \EE_{(s,a)\sim \mu}\sbr{f(s,a) g(s,a)}
        \\
        &= M(\Fcal) \norm{g}_\infty.
    \end{align*}
    Combining all the above, we get
    \begin{align*}
        \norm{\Qbar^X_{H^*}- B_\pi \Qbar^X_{H^*}}_{\Fcal^*}
        &\le 
        \norm{\Delta X}_{\Fcal^*} + \frac{M(\Fcal)}{H^*} \norm{Q^X_{1}-Q^X_{H^*+1}}_{\infty}
        \\
        &\le 
        \norm{\Delta X}_{\Fcal^*} + \frac{C}{H^*} M(\Fcal).
    \end{align*}
\end{proof}

\subsubsection{Proof of Proposition~\ref{prop:optimality_of_rm_fp}}
\label{sec:app_proof_optimality_of_rm_fp}

\begin{proof}
    The proof is similar to that of Proposition~\ref{prop:optimality_of_rm}.
    Let $c(X)\coloneqq \norm{\Qbar_{H^*}^X-B_\pi \Qbar_{H^*}^X}_2$
    and $\rho(X)\coloneqq \sqrt{\regretfp(\Qbar_{H^*}^X;\Xcal)}$.
    Then, by Lemma~\ref{lem:approximate_l2_dual_norm_with_regret},
    we have
    \begin{align*}
        \abs{c(X)-\rho(X)}\le
        \min_{A\in\Xcal} \norm{\Delta A}_2 + C\rbr{\frac{2\ln (2\abs{\Xcal}^2/\delta)}{n}}^{\frac14}
    \end{align*}
    with probability $1-\delta$ simultaneously all $X\in\Xcal$,
    which implies
    \begin{align*}
        c(\Xrmfp) - \min_{X\in\Xcal} c(X)
        &\le 
        \rho(\Xrmfp) - \min_{X\in\Xcal} \rho(X)
        +
        2\min_{X\in\Xcal} \norm{\Delta X}_2 + 2C\rbr{\frac{2\ln (2\abs{\Xcal}^2/\delta)}{n}}^{\frac14}
        \\
        &\le 
        2\min_{X\in\Xcal} \norm{\Delta X}_2 + 2C\rbr{\frac{2\ln (2\abs{\Xcal}^2/\delta)}{n}}^{\frac14},
    \end{align*}
    where the last inequality is owing to the definition of $\Xrmfp$.

    Note that
    $c(X)\le \norm{\Delta X}_2 + \frac{C}{H^*}M(L^2(\mu))$
    by Proposition~\ref{prop:metafqe_fp_guarantee}.
    Since $M(L^2(\mu))=1$~(see the proof of Proposition~\ref{prop:metafqe_fp_guarantee} for the definition of $M(\Fcal)$) and $H^*\ge n^{1/4}$,
    we have
    \begin{align*}
        \min_{X\in\Xcal} c(X) \le \min_{X\in\Xcal} \norm{\Delta X}_2 + \frac{C}{n^{1/4}}.
    \end{align*}
    Combining the above, we get
    \begin{align*}
        c(\Xrmfp)
        &\le
            3\min_{X\in\Xcal}\norm{\Delta X}_2
            + 3C\rbr{\frac{2\ln (2\abs{\Xcal}/\delta)}{n}}^{1/4}
    \end{align*}
    with probability $1-\delta$.
    Invoking Lemma~\ref{lem:metafqe_fp_error_bound}, we further get
    \begin{align*}
        \abs{\Delta J(Q^{\Xrmfp})}
        &\le \varphi_{L^2(\mu)}(c(\Xrmfp))
        \\
        &\le \varphi_{L^2(\mu)}\rbr{
                3\min_{X\in\Xcal} \norm{\Delta X}_2
                + 3C\rbr{\frac{2\ln (2\abs{\Xcal}/\delta)}{n}}^{1/4}
        }
        \\
        &\le \min_{X\in\Xcal}\crt_2(X) +
        \varphi_{L^2(\mu)}\rbr{
            3C\rbr{\frac{2\ln (2\abs{\Xcal}/\delta)}{n}}^{1/4}
        }
    \end{align*}
    where the last inequality follows from the monotonicity and the concavity of
    $\varphi_{L^2(\mu)}$.
    The desired result follows from
    $\varphi_\Fcal(y)\le C\norm{w}_\Fcal y$.
\end{proof}

\subsubsection{Proof of Proposition~\ref{prop:optimality_of_klm_fp}}
\label{sec:app_proof_optimality_of_klm_fp}
\begin{proof}
    Let $c(X)\coloneqq \norm{\Qbar_{H^*}^X-B_\pi \Qbar_{H^*}^X}_{\Fcal_\kappa^*}$
    and $\rho(X)\coloneqq \sqrt{\kblfp{\kappa}(X)}$.
    Then, by Lemma~\ref{lem:approximate_rkhs_dual_norm_with_kbl},
    we have
    \begin{align*}
        \abs{c(X)-\rho(X)}\le C\rbr{\frac{4(1\vee \ln (2\abs{\Xcal}/\delta))}{n}}^{\frac14}
    \end{align*}
    with probability $1-\delta$ simultaneously all $X\in\Xcal$,
    which implies
    \begin{align*}
        c(\Xklmfp{\kappa}) - \min_{X\in\Xcal} c(X)
        &\le 
        \rho(\Xklmfp{\kappa}) - \min_{X\in\Xcal} \rho(X)
        +
        2C\rbr{\frac{4(1\vee \ln (2\abs{\Xcal}/\delta))}{n}}^{\frac14}
        \\
        &\le 
        2C\rbr{\frac{4(1\vee \ln (2\abs{\Xcal}/\delta))}{n}}^{\frac14},
    \end{align*}
    where the last inequality is owing to the definition of $\Xklmfp{\kappa}$.

    Note that
    $c(X)\le \norm{\Delta X}_{\Fcal_\kappa^*} + \frac{C}{n^{1/4}}M(\Fcal_\kappa)$
    by Proposition~\ref{prop:metafqe_fp_guarantee}.
    Note also $M(\Fcal_\kappa)\le 1$ since $\kappa$ is normalized~(see the proof of Proposition~\ref{prop:metafqe_fp_guarantee} for the definition of $M(\Fcal)$).
    Thus,
    we have
    \begin{align*}
        \min_{X\in\Xcal} c(X) \le \min_{X\in\Xcal} \norm{\Delta X}_{\Fcal_\kappa^*} + \frac{C}{n^{1/4}},
    \end{align*}
    which implies
    \begin{align*}
        c(\Xklmfp{\kappa})
        &\le 
        \min_{X\in\Xcal} \norm{\Delta X}_{\Fcal_\kappa^*} + 
        3C\rbr{\frac{4(1\vee \ln (2\abs{\Xcal}/\delta))}{n}}^{\frac14}.
    \end{align*}
    Finally, the desired result is given by applying Lemma~\ref{lem:metafqe_fp_error_bound} to the LHS.
\end{proof}

\end{document}